\documentclass{article}
\usepackage{ijcai26}

\usepackage{times}
\usepackage{soul}
\usepackage{url}
\usepackage[hidelinks]{hyperref}
\usepackage[utf8]{inputenc}
\usepackage[small]{caption}
\usepackage{graphicx}
\usepackage{amsmath}
\usepackage{amsthm}
\usepackage{booktabs}
\usepackage{algorithm}
\usepackage{algorithmic}

\usepackage{subfigure}
\usepackage{makecell}
\usepackage{amssymb}
\usepackage{thm-restate}
\usepackage{enumerate}
\usepackage{xcolor}
\newtheorem{definition}{Definition}

\newcommand{\cV}{\mathcal{V}}

\newcommand{\sd}[1]{\ensuremath{_{\pm #1}}}
\newcommand{\best}[1]{\textbf{#1}}
\newcommand{\secondbest}[1]{\underline{#1}}

\newcommand{\SB}{{\sf Structural Balance}}
\newcommand{\SE}{{$Structure$ $Enhancement$}}
\newcommand{\RD}{{$Relation$ $Dif\!fusion$}}

\usepackage{thm-restate}
\usepackage{multirow}

\title{GraphSB: Boosting Imbalanced Node Classification on Graphs through Structural Balance}

\author{
Zhixiao Wang$^1$
\and
Chaofan Zhu$^1$
\and
Qihan Feng$^1$
\and
Jian Zhang$^1$
\and
Xiaobin Rui$^1$\thanks{Corresponding author.}
\and
Philip Yu$^2$
\affiliations
$^1$China University of Mining and Technology\\
$^2$University of Illinois Chicago
\emails
zhxwang@cumt.edu.cn,
chaofanz@cumt.edu.cn,
fengqh@cumt.edu.cn,
zhangjian10231209@cumt.edu.cn,
ruixiaobin@cumt.edu.cn,
psyu@uic.edu
}

\begin{document}

\maketitle

\begin{abstract}
Imbalanced node classification is a critical challenge in graph learning, where most existing methods typically utilize Graph Neural Networks (GNNs) to learn node representations.
These methods can be broadly categorized into the data-level and the algorithm-level.
The former aims to synthesize minority-class nodes to mitigate quantity imbalance, while the latter tries to optimize the learning process to highlight minority classes.
However, neither of them addresses the inherently imbalanced graph structure, which is a fundamental factor that incurs majority-class dominance and minority-class assimilation in GNNs. 
Our theoretical analysis further supports this critical insight.
Therefore, we propose GraphSB (Graph Structural Balance), a novel framework that incorporates {\SB} as a key strategy to address the underlying imbalanced graph structure before node synthesis.
{\SB} performs a two-stage structure optimization: {\SE} that mines hard samples near decision boundaries through dual-view analysis and enhances connectivity for minority classes through adaptive augmentation, and {\RD} that propagates the enhanced minority context while simultaneously capturing higher-order structural dependencies.
Thus, GraphSB balances structural distribution before node synthesis, enabling more effective learning in GNNs.
Extensive experiments demonstrate that GraphSB significantly outperforms the state-of-the-art methods.
More importantly, the proposed {\SB} can be seamlessly integrated into state-of-the-art methods as a simple plug-and-play module, increasing their accuracy by an average of $4.57\%$.
\end{abstract}

\section{Introduction}\label{sec:intro}

In recent years, Graph Neural Networks (GNNs)~\cite{1} have become a powerful tool for graph representation learning, achieving advanced performance in various graph-related downstream tasks~\cite{2,3,4}. 
Among them, node classification~\cite{5,6} is one of the most fundamental tasks. 
The goal of node classification is to predict the labels of the unlabeled nodes based on the characteristics of the nodes.

In fact, real-world graphs often exhibit a significant class imbalance, with the number of samples varying greatly across different classes~\cite{9,liu2023survey}. 
Thus, the classification of unbalanced nodes has become a critical challenge. 
To address this problem, researchers have explored various GNN-based approaches, which can be broadly categorized into two types: data-level methods and algorithm-level methods.
Data-level methods aim to balance the training process by synthesizing additional samples for the minority class. 
To highlight minority-class nodes, algorithm-level methods attempt to optimize the learning process.
Despite their contributions, they are essentially palliative as they fail to directly tackle the underlying imbalanced structure.

In practice, the imbalanced graph structure usually exhibits strong asymmetry: minority-class nodes tend to have sparser neighborhoods.
This induces biased neighbor aggregation during message passing, severely degrading the representation quality of minority-class nodes.
Therefore, failing to address fundamental structural imbalance not only limits the benefits of quantity balancing techniques, but also perpetuates structural bias throughout the learning process, resulting in a final biased representation.

Figure~\ref{lizi} illustrates how structural imbalance hinders the learning of minority classes. 
Due to the scarcity of edges connecting minority-class nodes, these nodes receive insufficient same-class context during neighborhood aggregation. 
Consequently, their latent representations are progressively compressed toward the majority-class subspace. 
This compression effect is particularly harmful to the \textbf{hard samples} located near decision boundaries, {\em i.e.}, nodes that naturally lie between class regions and rely on stronger same-class support to maintain correct classification. 
Furthermore, as GNN depth increases, minority-class context decays exponentially, while majority-class context accumulates through repeated aggregation, further shrinking decision boundaries.
This phenomenon arises fundamentally from the intrinsic coupling between node embeddings and graph structure, highlighting the need for structural intervention at early stages.

\begin{figure}[!h]
\centering
\includegraphics[width=\linewidth]{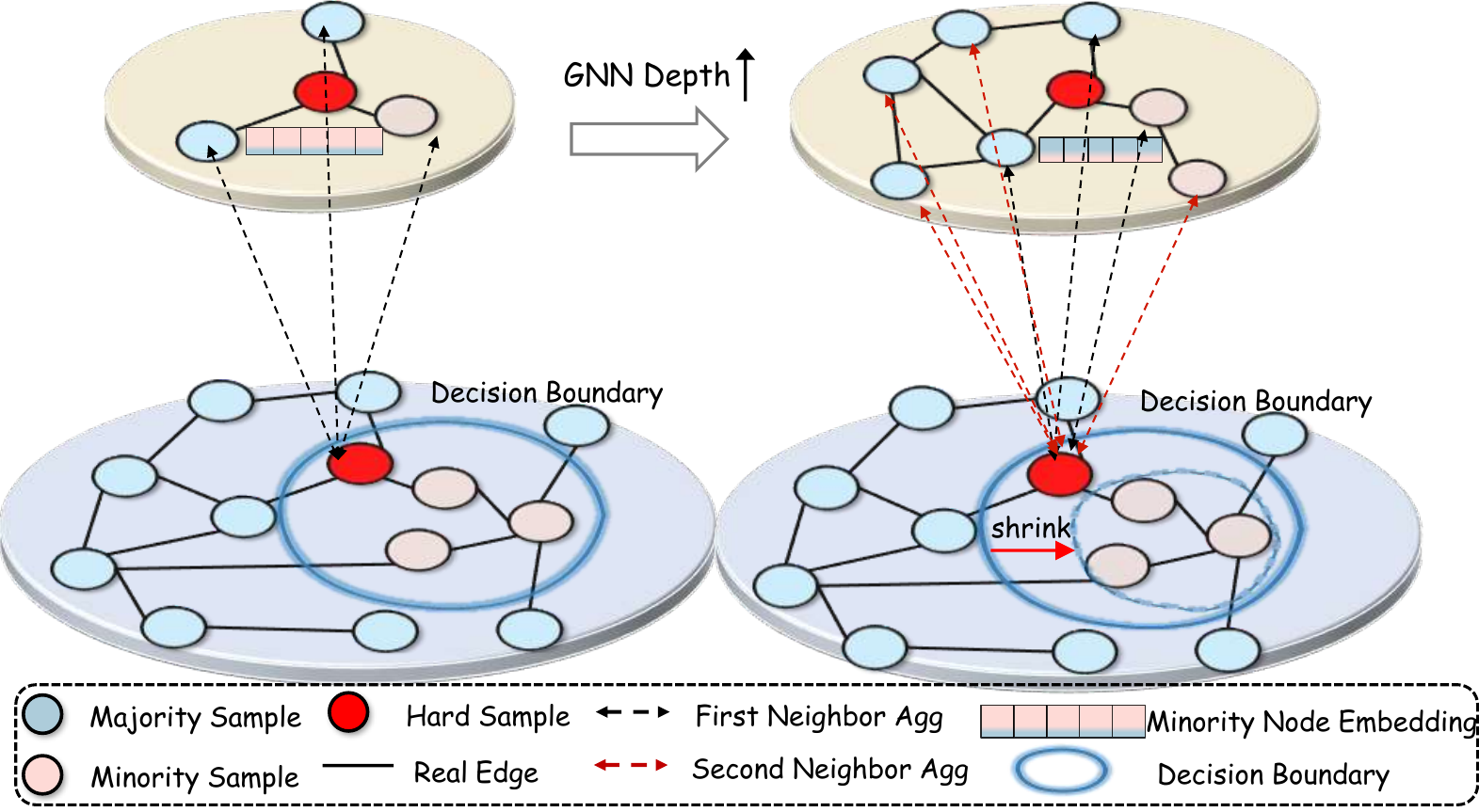}
\caption{Imbalanced node learning on graphs}
\label{lizi}
\end{figure}


Motivated by these observations, we propose a comprehensive imbalanced node classification framework named \textbf{GraphSB} (Graph Structural Balance) to address the inherent imbalanced graph structure. 
Our contributions are summarized as follows:
\begin{itemize}
\setlength{\itemsep}{0pt}
\setlength{\parsep}{0pt}
    \item \textbf{Theoretical Analysis}. 
    Our theoretical analysis reveals that the imbalanced graph structure leads GNNs to suffer from majority-class dominance and minority-class assimilation, causing biased representation of the minority-class nodes.
    Such bias will accumulate in subsequent node synthesis and learning stages, resulting in significantly degraded representation quality.

    \item \textbf{Model Design}. We propose GraphSB, which incorporates {\SB} as a key strategy to address imbalanced graph structure.
    {\SB} utilizes a two-stage optimization: (1) {\SE} mines hard samples near decision boundaries via dual-view analysis and strengthens their minority-class connectivity, while (2) {\RD} 
    propagates the enhanced minority context through multi-step iterative diffusion while capturing higher-order structural dependencies.
    
    \item \textbf{Experimental Evaluation}. Experiments on eight datasets demonstrate that GraphSB significantly outperforms the state-of-the-art methods. 
    More importantly, {\SB} can be seamlessly integrated into state-of-the-art methods as a simple plug-and-play module, increasing their accuracy by $4.57\%$ on average.
\end{itemize}

\section{Related Work}\label{sec:rela}

\subsection{Data-level Methods}

Data-level methods are dedicated to alleviating this quantity imbalance across different classes~\cite{ding2022data}. 
For example, GraphENS~\cite{9} synthesizes minority-class nodes by incorporating information from majority classes. 
GraphSMOTE~\cite{zhao2021graphsmote} extends the idea of SMOTE~\cite{chawla2002smote} to generate synthetic nodes.
GraphMixup~\cite{13} enhances Mixup~\cite{14} to synthesize minority-class nodes in the embedding space and subsequently determines the neighborhoods of the newly generated nodes with a pre-trained edge predictor. 
GraphSHA~\cite{GraphSHA} aims to synthesize harder minority samples.
Although these methods increase the number of minority-class samples, they focus mainly on quantity rebalancing while overlooking the inherent structural imbalance of graphs, fundamentally limiting their effectiveness.


\subsection{Algorithm-level Methods}

Algorithm-level methods mitigate class imbalance by optimizing the learning process through penalty functions, model regularization, or other sophisticated techniques.
For example, DR-GCN~\cite{DR-GCN} tackles the class-imbalanced problem by imposing two types of regularization, combining conditional adversarial training with latent distribution alignment. 
ReNode~\cite{ReNode} incorporates modified loss functions or designed penalty terms to emphasize minority classes. TAM~\cite{TAM} adjusts classification margins using local neighbor distributions but ignores global structural dominance.
BAT~\cite{liu2024class} augments the weights of minority-class nodes to alleviate ambivalent and long-range message passing, while IceBerg~\cite{Iceberg} leverages debiased self-training and decoupled propagation to reach distant unlabeled nodes. Despite these efforts, such methods primarily optimize the learning process to obtain a better embedding space for minority-class nodes, yet fail to address the fundamental imbalanced graph structure.



\section{Structural Imbalance in GNNs}
\label{sec:theoretical}

\subsection{Preliminaries}
Let $\mathcal{G}=(V,E,X)$ denote an attributed graph with node set $V$, edge set $E$, and feature matrix $X \in \mathbb{R}^{n \times p_0}$, where $n = |V|$ and $x_i \in \mathbb{R}^{p_0}$ represents the feature vector of node $i \in \{1,\cdots,n\}$. 
In message passing neural networks (MPNN)~\cite{MPNNs}, the node representation $h_i^{(\ell)} \in \mathbb{R}^{p_\ell}$ at layer $\ell \geq 0$ evolves through message passing. 
Let $\mathcal{N}(i)$ denote node $i$'s neighborhood, $\phi_\ell$, $\psi_\ell$ refer to the learnable update function and the message aggregation function, respectively. 
The layer-wise propagation rule is given by:
\begin{equation}
h_i^{(\ell+1)} = \phi_\ell\left(h_i^{(\ell)}, \sum_{j \in \mathcal{N}(i)} \hat{A}_{ij} \psi_\ell(h_i^{(\ell)}, h_j^{(\ell)})\right),
\end{equation}
where $\hat{A} = \tilde{D}^{-1/2}\tilde{A}\tilde{D}^{-1/2}$ refers to the normalized adjacency matrix with self-loops. $\tilde{A} = A + I$ denotes the adjacency matrix with self-connections and $\tilde{D}$ represents the degree matrix of $\tilde{A}$.  

Consider a class-imbalanced graph $\mathcal{G}$ with two classes, {\em i.e.}, $V = V_\text{mi}\cup V_\text{ma}$, where $V_\text{mi}$ and $V_\text{ma}$ refer to the nodes of the minority and majority class, respectively.
Let $n_1 = |V_\text{ma}|$ and $n_2 = |V_\text{mi}|$, the imbalance ratio is $\beta = n_1/n_2 \gg 1$. 
Under a Stochastic Block Model~\cite{SBM} SBM$(n, p, q)$ with intra-class edge probability $p$ and inter-class edge probability $q$, where $p \gg q$, the expected degrees for a minority-class node and a majority-class node are:
\begin{itemize}
    \item minority-class: $\mathbb{E}[d_\text{mi}] = n_1 p + n_2 q = n_1 (p + q \beta)$
    \item majority-class: $\mathbb{E}[d_\text{ma}] = n_1 q + n_2 p = n_1 (q + p \beta)$
\end{itemize}

Let $\tau = \frac{\mathbb{E}[d_\text{ma}]}{\mathbb{E}[d_\text{mi}]} = \frac{q + p\beta}{p + q\beta} > 1$ denote the degree disparity that quantifies the structural advantage of majority-class nodes, we then analyze how the structural imbalance incurs majority-class dominance and minority-class assimilation.

\subsection{Majority-class Dominance}

In class imbalance, two key mechanisms disadvantage minority-class nodes: information dilution through over-squashing and gradient dominance from the majority class.

\textbf{The over-squashing phenomenon} in GNNs, first analyzed by Alon {\em et al.} \cite{bottleneck}, manifests as exponential decay of information from distant nodes. Consider an MPNN with node features $X \in \mathbb{R}^{n\times p_0}$, learnable update function $\phi_\ell$, and message aggregation function $\psi_\ell$, satisfying $\|\nabla\phi_\ell\| \leq \eta_1$, $\|\nabla\psi_\ell\| \leq \eta_2$ where $\eta_1$ and $\eta_2$ are two constants. 
For nodes $u, v\in V$ with geodesic distance $d_\mathcal{G}(u,v) = r+1$, the Jacobian sensitivity is as follows:

\begin{equation}
\left|\frac{\partial h_u^{(r+1)}}{\partial x_v}\right| \leq (\eta_1\eta_2)^{r+1}(\hat{A}^{r+1})_{uv}
\label{eq:jacobian_bound}.
\end{equation}

In bottleneck structures such as $b$-ary trees, $(\hat{A}^{r+1})_{uv} = \mathcal{O}(b^{-r})$ confirms exponential information distortion.

\begin{definition}
\textbf{Information Dilution} is the phenomenon where information propagation between nodes decays rapidly as the distance $L$ increases. In class-imbalanced graphs, this decay is further accelerated by the degree disparity $\tau$, with the attenuation factor amplified by $\tau^{-L}$.
\end{definition}

The relationship between information dilution and the over-squashing phenomenon is formally quantified in the following Theorem~\ref{thm:info_dilution} (detailed proof given in Appendix~\ref{app:proofs}).

\begin{restatable}{theorem}{theoremone}
\label{thm:info_dilution}
For a path of length $L$ in a class-imbalanced graph with imbalance ratio $\beta$ and degree disparity $\tau$, the expected path weight satisfies:
\begin{equation}
\mathbb{E}[\mathcal{W}_L] = \frac{W^{(L)}}{\mathbb{E}[d_\text{mi}]^L}\left(\frac{\beta+\tau}{\tau(\beta+1)}\right)^L,
\end{equation}
where $\mathcal{W}_L$ denotes the weight of information propagated along a path of length $L$. $W^{(L)} = \prod_{\ell=1}^L \nabla\phi_\ell \nabla\psi_\ell$ represents the cumulative weight matrix product.
\end{restatable}



Theorem~\ref{thm:info_dilution} demonstrates that the majority-class nodes accelerate information dilution, causing node features of the minority-class to be overwhelmed in message passing.

To better characterize the mechanism where the majority-class nodes dominate the learning process, we formally introduce the concept of Gradient Dominance as follows.

\begin{definition}
\textbf{Gradient Dominance} is the phenomenon in class-imbalanced graphs where the gradients associated with majority-class nodes overwhelmingly influence the parameter updates during training, thereby suppressing the learning information from minority-class nodes.
\end{definition}

The effect of Gradient Dominance is formally quantified in Theorem~\ref{thm2} (detailed proof given in Appendix~\ref{app:proofs}).

\begin{restatable}{theorem}{theoremtwo}
\label{thm2}
Assume path weight $\mathcal{W}_L$ is independent of node degrees, then the expected ratio of gradient magnitudes between a minority-class node $u \in V_\text{mi}$ and a majority-class node $v \in V_\text{ma}$ at depth $L$ is:
\begin{align}
\frac{\mathbb{E}[\|\nabla_\theta L(u)\|]}{\mathbb{E}[\|\nabla_\theta L(v)\|]} &\propto 
\frac{1}{\beta}.
\end{align}
\end{restatable}





Theorem~\ref{thm2} demonstrates that the numerical advantage and higher degree of majority-class nodes make them dominant in information propagation and gradient updates.

\subsection{Minority-class Assimilation}


\begin{definition}
\textbf{Minority-class Assimilation} refers to the progressive convergence of minority-class nodes' representations toward the centroids of majority classes during GNN message passing.
As a result, the distinctive characteristics of the minority-class nodes are overwhelmed by the predominant influence of the majority-class nodes, leading to reduced class separability in the deeper layers.
\end{definition}

The following Theorem~\ref{thm:feature_conv} reveals how the node features of minority classes assimilate to those of the majority classes as the network depth increases.

\begin{restatable}{theorem}{theoremthree}
\label{thm:feature_conv}
In class-imbalanced graphs with imbalance ratio $\beta \gg 1$, the centroid distance $\Delta^{(\ell)} = \mu_\text{mi}^{(\ell)} - \mu_\text{ma}^{(\ell)}$ between minority and majority classes decays exponentially with network depth, which holds:
\begin{equation}
\|\Delta^{(\ell)}\| \leq C \cdot (\sigma_{\max}(W)\lambda_2(M))^\ell\|\Delta^{(0)}\|,
\end{equation}
where $C$ is a constant, $\Delta^{(0)}$ refers to the initial centroid distance, $\sigma_{\max}(W)$ represents the maximum singular value of weight matrix $W$, and $\lambda_2(M)$ denotes the second eigenvalue of the propagation matrix $M$ with $\sigma_{\max}(W)\lambda_2(M) < 1$.
\end{restatable}


For example, let $p = 0.5$, $q = 0.1$, and $\beta = 10$, we have
\[
M = \begin{bmatrix}
0.333 & 0.196 \\
0.067 & 0.980
\end{bmatrix},
\quad \lambda_2 \approx 0.313.
\]

Since $\lambda_2 < 1$, it governs the convergence rate. 
By spectral decomposition, the magnitude of $\Delta^{(\ell)}$ decays as $\mathcal{O}(0.313^\ell)$, showing how rapidly the node features of minority classes assimilate with the increase of network depth.



\section{Methodology}
\label{sec4}

\begin{figure*}[!ht]
\centering
\includegraphics[width=\linewidth]{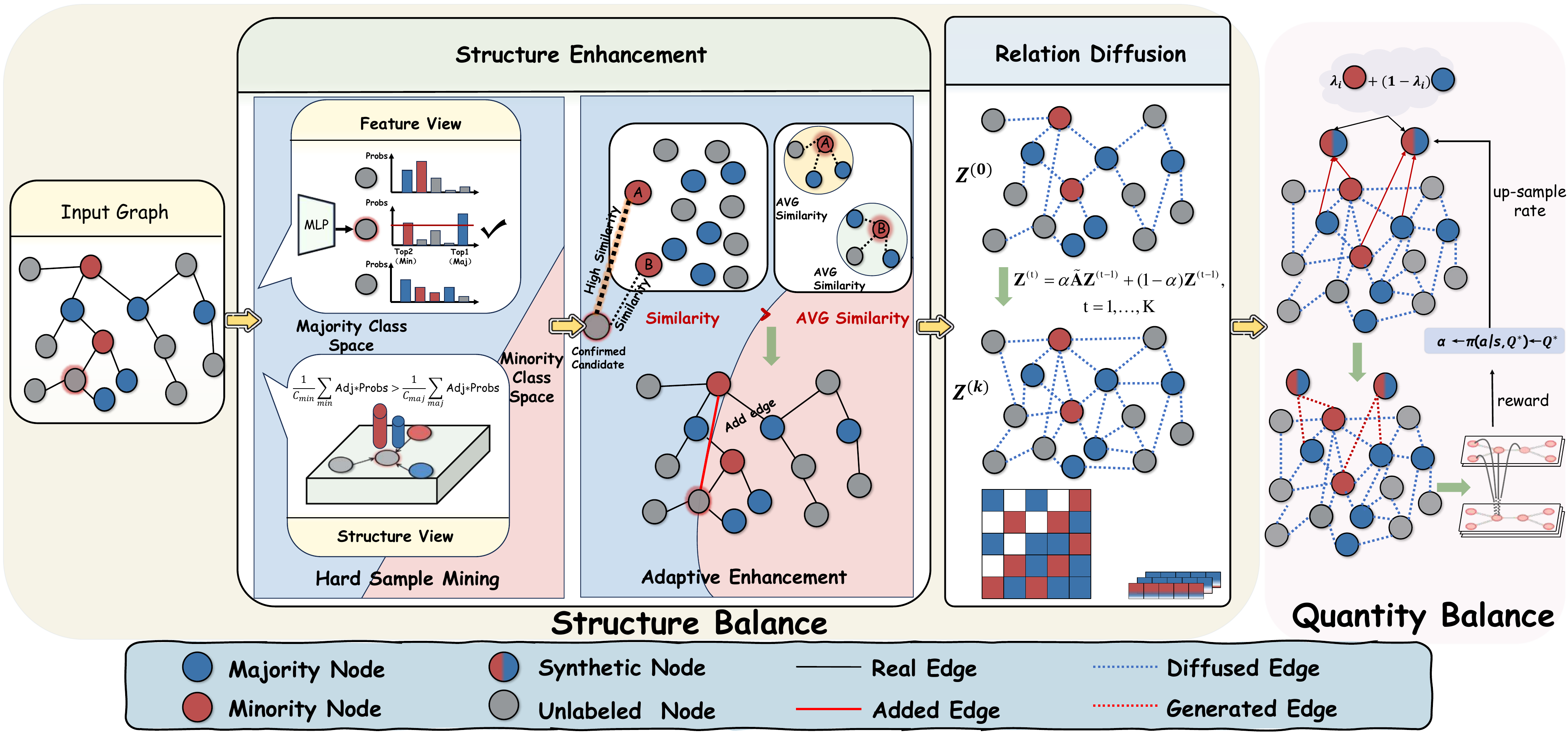}
\caption{Overview of the GraphSB framework.}
\label{framework}
\end{figure*}

Building on the theoretical analysis presented in the above section, we propose GraphSB (Graph Structural Balance) to explicitly address the inherent structural imbalance in graphs. 
As illustrated in Figure~\ref{framework}, {\SB} involves a two-stage optimization: (1) {\SE} mines hard samples and enhances minority connectivity through adaptive augmentation, and (2) {\RD} propagates the enhanced minority context while capturing higher-order structural dependencies. 
For the quantity balance, we follow the same approach as GraphMixup~\cite{13}.
It is worth noting that our proposed {\SB} can be integrated into any state-of-the-art methods besides GraphMixup (Complexity analysis is detailed in the Appendix~\ref{app:complexity}).

\subsection{Structure Enhancement}  
In \textbf{imbalanced graphs}, the scarcity of minority-class edges hampers minority nodes from receiving sufficient same-class context during aggregation.
We therefore propose {\SE}, which mines hard samples near decision boundaries through dual-view analysis and enhances minority connectivity via adaptive augmentation.

\textbf{(1) Hard Sample Mining:} We employ dual-view (feature view vs. neighbor view) analysis to mine hard samples that exhabit \textbf{feature-neighbor conflicts}, where a node's feature-based prediction diverges from its neighborhood consensus. Such conflicts reveal samples suffering from minority subspace compression, which are critical for recovering blurred decision boundaries through targeted structural correction.

\textbf{\textit{Feature view:}} We leverage a lightweight MLP encoder to project input features into a latent embedding space and obtain node-level predictions via softmax:
\begin{equation}
\mathbf{h}_u = \text{MLP}(\mathbf{x}_u) \in \mathbb{R}^{d},
\end{equation}
\begin{equation}
P_{\text{self}}(u) = \text{Softmax}(\mathbf{h}_u) \in \mathbb{R}^{|\mathcal{C}|},
\end{equation}
where $\mathcal{C}$ denotes the set of node classes. 

Let $\hat{y}^{(1)}_u$ and $\hat{y}^{(2)}_u$ denote node $u$'s top-1 and top-2 predicted classes.
A node $u$ with $\hat{y}^{(1)}_u$ indicating majority-class while $\hat{y}^{(2)}_u$ indicating minority-class is considered to lie near decision boundaries where the minority subspace is compressed by the majority class.
These nodes are identified by:
\begin{equation}
\mathcal{S}_{\text{init}} = \left\{ u \in \mathcal{V} \mid \hat{y}^{(1)}_u \in \mathcal{C}_{\text{ma}} \land \hat{y}^{(2)}_u \in \mathcal{C}_{\text{mi}} \land P_{\text{self}}(\hat{y}^{(2)}_u) > \xi \right\},
\end{equation}
where $\mathcal{C}_{\text{ma}}$ and $\mathcal{C}_{\text{mi}}$ denote majority and minority class sets, Given the $0.5$ bound for the second-highest probability, we set $\xi$ to half this limit to filter noise while retaining valid signals.

\textbf{\textit{Neighbor View:}} We aggregate neighbor predictions via soft voting to validate  consensus and compute the average neighborhood prediction for minority and majority classes as:
\begin{equation}
P_{\text{neigh}}(u) = \sum_{v \in \mathcal{N}(u)} \hat{A}_{uv} \cdot P_{\text{self}}(v),
\end{equation}
\begin{equation}
\bar{P}_{\text{neigh}}(u, \mathcal{C}) = \frac{1}{|\mathcal{C}|} \sum_{c \in \mathcal{C}} P_{\text{neigh}}(u)[c],
\end{equation}
where $\hat{A}$ denotes the normalized adjacency matrix. 

The final candidate set requires neighborhood consensus to favor minority classes, confirming these nodes as genuine minority members worth augmenting as:
\begin{equation}
\mathcal{S}_{\text{cand}} = \left\{ u \in \mathcal{S}_{\text{init}} \mid \bar{P}_{\text{neigh}}(u, \mathcal{C}_{\text{mi}}) > \bar{P}_{\text{neigh}}(u, \mathcal{C}_{\text{ma}}) \right\}.
\end{equation}

Generally, this dual-view screening mines hard samples that exhibit feature-neighbor conflict: feature predictions lean toward majority while neighborhood consensus supports minority. 
These validated minority neighborhoods also provide high-quality connections for subsequent relation diffusion.

\textbf{(2) Adaptive Augmentation:} After identifying these hard samples, we connect them to minority-class anchors in the training set $\mathcal{T}$ to strengthen minority context in their local neighborhoods. 
To prevent incorrect edge construction and preserve homophily, we impose a neighbor-level similarity constraint using the learned embeddings $\mathbf{h}$ as:
\begin{equation}
\tau_v = \frac{1}{|\mathcal{N}(v)|} \sum_{w \in \mathcal{N}(v)} \text{sim}(\mathbf{h}_v, \mathbf{h}_w),
\end{equation}
where $\text{sim}(\cdot, \cdot)$ denotes cosine similarity. 

For each candidate $u\in \mathcal{S}_{\text{cand}}$, we identify the most similar anchor within its predicted minority class $c^* = \hat{y}^{(2)}_u$ as:
\begin{equation}
v^* = \underset{v \in \mathcal{T} \cap \{c^*\}}{\text{argmax}} \; \text{sim}(\mathbf{h}_u, \mathbf{h}_v).
\end{equation}

Critically, an edge $(u, v^*)$ is added only if the candidate is more similar to the anchor than the anchor's average neighbors, {\em i.e.}, $\text{sim}(\mathbf{h}_u, \mathbf{h}_{v^*}) > \tau_{v^*}$.
The augmented edge set is:
\begin{equation}
E_{\text{aug}} = E \cup \left\{ (u, v^*) \mid u \in \mathcal{S}_{\text{cand}}, \text{sim}(\mathbf{h}_u, \mathbf{h}_{v^*}) > \tau_{v^*} \right\}.
\end{equation}

Through dual-view analysis, {\SE} identifies hard samples near decision boundaries that are crucial for recovering blurred class separations. 
The adaptive augmentation ensures edge quality by requiring source nodes to be more similar to anchors than their existing neighbors, thereby maintaining graph homophily while providing high-quality minority connections for subsequent propagation.

\subsection{Relation Diffusion}
While {\SE} effectively repairs the local topology, standard graph convolution remains limited to local neighborhoods, making it difficult to stably capture long-range dependencies.
To overcome this limitation, we further refine the connectivity structure utilizing a sparse iterative {\RD} designed to globally propagate the enriched minority context, thereby capturing high-order structural dependencies beyond immediate neighbors.

First, based on the augmented adjacency matrix $\mathbf{A}'$ obtained from {\SE}, we compute the symmetrically normalized matrix $\mathbf{\hat{A}} = \mathbf{D}^{-\frac{1}{2}}(\mathbf{A}' + \mathbf{I})\mathbf{D}^{-\frac{1}{2}}$, where $\mathbf{I}$ denotes the self-loop identity matrix.
We adopt an efficient iterative propagation strategy.
Specifically, we first project the input features into a latent space:
\begin{equation}
    \mathbf{Z}^{(0)} = \mathbf{X}\mathbf{W} + \mathbf{b},
\end{equation}
where $\mathbf{W}$ and $\mathbf{b}$ are learnable parameters.

Next, to enhance robustness and prevent overfitting to specific structural noise, we incorporate a \textbf{Stochastic Structure Perturbation} ({\em i.e.}, edge dropout) strategy \textit{before} propagation.
We generate a perturbed sparse adjacency matrix $\mathbf{\tilde{A}}$ by randomly masking edges from $\mathbf{\hat{A}}$ with a dropout rate $p_{\text{drop}}$.
Formally, for each edge $(i,j) \in E_{\text{aug}}$, we sample a Bernoulli mask $m_{ij} \sim \text{Bernoulli}(1-p_{\text{drop}})$ and re-scale the weights:
\begin{equation}
    \mathbf{\tilde{A}}_{ij} = \mathbf{\hat{A}}_{ij} \cdot m_{ij} \cdot (1-p_{\text{drop}})^{-1}.
\end{equation}

Finally, we perform $K$-step sparse iterative relation diffusion on the projected features using the perturbed structure. This process acts as a low-pass filter to smooth node embeddings within the repaired same-class neighborhoods:
\begin{equation}
    \mathbf{Z}^{(t)} = \alpha \mathbf{\tilde{A}}\mathbf{Z}^{(t-1)} + (1-\alpha)\mathbf{Z}^{(t-1)}, \quad t=1,\dots,K,
\end{equation}
where the coefficient $\alpha \in (0,1)$ controls the trade-off between local preservation and neighborhood information gain.



Finally, we inject the bias term and apply the activation function with feature dropout to obtain final representations:
\begin{equation}
    \mathbf{H} = \text{dropout}\left(\sigma\left(\mathbf{Z}^{(K)} + \mathbf{b}\right), p\right),
\end{equation}
where $\mathbf{b}$ is the bias vector, $\sigma$ denotes the nonlinear activation function (e.g., ReLU), and $p$ is the feature dropout rate.
This formulation avoids dense matrix operations, ensuring linear complexity $\mathcal{O}(K|E|d)$ while effectively propagating minority context globally.

\subsection{Objective Optimization and Training Strategy}
\label{sec3.5}

After obtaining high-quality node embeddings via {\SB}, we employ a GNN-based classifier as: 
\begin{equation}
h_v^{(L+1)} = \sigma\left(\mathbf{\widetilde{W}}^{(1)} \cdot [h_v^{(L)} \parallel \mathbf{H}_O^{(L)} \cdot \mathbf{A}_O[:, v]]\right),
\end{equation}
\begin{equation}
\hat{\mathbf{y}}_v = \text{softmax}\left(\mathbf{\widetilde{W}}^{(2)} \cdot h_v^{(L+1)}\right),
\end{equation}
where $\parallel$ denotes the concatenation operation, $\mathbf{\widetilde{W}}^{(1)} \in \mathbb{R}^{F \times 2F}$ and $\mathbf{\widetilde{W}}^{(2)} \in \mathbb{R}^{M \times F}$ are both learnable weight matrices, and $\mathbf{A}_O[:, v]$ represents the connectivity pattern of node $v$. This architecture enhances node classification by incorporating both node features and structural information.

The classifier is optimized using cross-entropy loss on the combined set of original and synthetic labeled nodes as:
\begin{equation}
\mathcal{L}_{\text{node}} = \frac{1}{|\cV_O|} \sum_{v \in \cV_O} \mathcal{L}_{CE}(\hat{\mathbf{y}}_v, y_v),
\end{equation}
where $\cV_O$ represents the set of labeled nodes including synthesized nodes and $\mathcal{L}_{CE}$ is the cross-entropy loss function. 

The overall training objective of GraphSB integrates structure balance, node synthesis, and classification. 
The final objective function can be defined as
\begin{equation}
\min_{\theta, \gamma, \phi} \mathcal{L}_{\text{node}} + \mathcal{L}_{\text{edge}},
\end{equation}
where $\theta$, $\gamma$, and $\phi$ represent the parameters for the feature extractor, edge predictor, and node classifier, respectively. 
$\mathcal{L}_{\text{edge}}$ denotes the edge reconstruction loss.

Since this paper utilizes GraphMixup as the quantity balance module, we adopt the same two-phase training strategy.
First, the feature encoder and edge generation module undergo structural initialization through exclusive optimization of $\mathcal{L}_{\text{edge}}$. 
Then, this pre-trained configuration serves as the foundation for a holistic network refinement, where the composite $\mathcal{L}_{\text{node}} + \mathcal{L}_{\text{edge}}$ drives joint parameter updates.

\section{Experiments}
\label{sec5}



\subsection{Experimental Setup}
\label{sec5.1}

\subsubsection{Datasets}
\label{sec4.1.1}
We evaluate GraphSB on eight benchmark datasets: three citation networks (Cora, Citeseer, PubMed)~\cite{sen2008collective}, two e-commerce networks (Amazon-Photo, Amazon-Computers)~\cite{huang2017label}, one co-authorship network (Wiki-CS)~\cite{mernyei2020wiki}, and two large-scale graphs (ogbn-arxiv, CoraFull) that exhibit natural extreme class imbalance. Table~\ref{tab:dataset_stats} summarizes the key statistics, where $\rho_0$ denotes the naive imbalance ratio between the size of the smallest class and that of the largest class. 
\begin{table}[h]
\centering
\scriptsize
\setlength{\tabcolsep}{1.8mm}
\caption{Dataset statistics.}
\vspace{-5pt}
\label{tab:dataset_stats}
\begin{tabular}{lrrrrr}
\toprule
Dataset & Nodes & Edges & Features & Classes& $\rho_0$  \\
\toprule
Cora & 2,708 & 10,556 & 1,433 & 7 & 0.22 \\
Citeseer & 3,327 & 9,104 & 3,703 & 6 & 0.36 \\
PubMed & 19,717 & 88,648 & 500 & 3 & 0.30 \\
Photo & 7,650 & 238,162 & 500 & 8 & 0.17 \\
Computers & 13,752 & 491,722 & 500 & 10 & 0.056 \\
Wiki-cs & 11,701 & 216,123 & 300 & 10 & 0.35 \\
ogbn-arxiv & 169,243 & 1,166,243 & 129 & 40 & 0.001 \\
CoraFull & 19,793 & 126,842 & 8,710 & 70 & 0.016 \\
\bottomrule
\end{tabular}
\end{table}

\begin{table*}[!t]
\centering
\scriptsize 
\setlength{\tabcolsep}{1.5pt} 
\caption{Performance comparison on six datasets (Acc and Macro-F1). The standard deviation is formatted as a subscript.}
\label{tab:merged_results_subscript}
\vspace{-5pt}
\begin{tabular*}{\textwidth}{@{\extracolsep{\fill}}lcccccccccccc@{}}
\toprule
\multirow{2}{*}{Methods} & \multicolumn{2}{c}{Cora} & \multicolumn{2}{c}{Citeseer} & \multicolumn{2}{c}{PubMed} & \multicolumn{2}{c}{Photo} & \multicolumn{2}{c}{Computers} & \multicolumn{2}{c}{Wiki-CS} \\
\cmidrule(lr){2-3} \cmidrule(lr){4-5} \cmidrule(lr){6-7} \cmidrule(lr){8-9} \cmidrule(lr){10-11} \cmidrule(lr){12-13}
 & Acc & Macro-F1 & Acc & Macro-F1 & Acc & Macro-F1 & Acc & Macro-F1 & Acc & Macro-F1 & Acc & Macro-F1 \\
\midrule
Vanilla      
& 68.12\sd{0.73} & 68.14\sd{0.55} & 54.62\sd{0.61} & 53.75\sd{0.54} & 67.51\sd{1.04} & 67.43\sd{1.12} 
& 81.23\sd{0.82} & 81.34\sd{0.75} & 80.45\sd{0.71} & 79.92\sd{0.68} & 76.79\sd{0.54} & 73.56\sd{0.46} \\

Re-weight    
& 72.92\sd{1.24} & 72.43\sd{1.35} & 55.52\sd{0.53} & 55.73\sd{0.52} & 69.72\sd{0.33} & 69.42\sd{0.33} 
& 84.24\sd{1.15} & 84.07\sd{1.13} & 81.12\sd{0.48} & 80.51\sd{0.79} & 76.12\sd{0.48} & 73.85\sd{0.40} \\

SMOTE        
& 73.24\sd{1.32} & 73.02\sd{1.32} & 54.03\sd{0.62} & 53.74\sd{0.63} & 70.82\sd{0.44} & 70.43\sd{0.33} 
& 84.12\sd{1.37} & 83.89\sd{1.43} & 81.01\sd{1.00} & 80.43\sd{1.03} & 78.07\sd{0.31} & 74.58\sd{0.32} \\

Embed-SMOTE  
& 72.32\sd{0.97} & 72.23\sd{0.85} & 56.73\sd{0.64} & 56.82\sd{0.63} & 69.23\sd{0.32} & 69.13\sd{0.32} 
& 79.54\sd{0.74} & 79.95\sd{0.78} & 80.06\sd{0.68} & 80.35\sd{0.69} & 75.03\sd{0.42} & 72.19\sd{0.42} \\

DR-GNN       
& 74.82\sd{1.12} & 74.52\sd{0.96} & 55.23\sd{1.28} & 55.83\sd{1.21} & 68.52\sd{1.05} & 68.14\sd{0.92} 
& 80.34\sd{0.92} & 80.48\sd{0.86} & 81.23\sd{0.78} & 80.87\sd{0.73} & 78.69\sd{0.28} & 75.78\sd{0.47} \\

TAM     
& 75.53\sd{1.26} & 75.52\sd{1.08} & 58.43\sd{1.62} & 58.32\sd{1.44} & 71.32\sd{1.18} & 71.52\sd{1.02} 
& 81.67\sd{0.88} & 81.52\sd{0.82} & 81.94\sd{0.69} & 81.36\sd{0.71} & 79.08\sd{0.30} & 76.45\sd{0.49} \\

GraphSMOTE   
& 74.32\sd{1.08} & 73.92\sd{0.93} & 57.52\sd{1.44} & 56.72\sd{1.26} & 68.23\sd{1.02} & 68.23\sd{0.91} 
& 82.21\sd{0.44} & 82.15\sd{0.35} & 82.62\sd{0.45} & 80.15\sd{0.57} & 78.56\sd{0.34} & 75.29\sd{0.35} \\

GraphENS     
& 75.82\sd{0.43} & 74.13\sd{0.32} & 59.52\sd{0.52} & 58.23\sd{0.43} & 72.23\sd{0.23} & 71.92\sd{0.23} 
& 83.58\sd{0.67} & 83.42\sd{0.61} & 82.45\sd{0.58} & 81.98\sd{0.64} & 79.27\sd{0.18} & 76.16\sd{0.43} \\

GraphMixup   
& 77.52\sd{1.21} & 77.42\sd{1.05} & 63.32\sd{1.63} & 63.12\sd{1.47} & 71.42\sd{1.15} & 71.72\sd{1.03} 
& 84.17\sd{1.14} & 84.06\sd{1.05} & 82.09\sd{0.64} & 81.59\sd{0.81} & \secondbest{80.45}\sd{0.26} & 76.59\sd{0.25} \\

GraphSHA     
& 78.35\sd{0.82} & 77.78\sd{0.67} & 63.40\sd{0.54} & 62.68\sd{0.52} & \secondbest{77.11}\sd{0.42} & 77.45\sd{0.77} 
& 84.52\sd{0.76} & 84.37\sd{0.69} & 82.78\sd{0.62} & \secondbest{82.14}\sd{0.58} & 79.84\sd{0.35} & 76.52\sd{0.38} \\

CDCGAN       
& 78.12\sd{0.42} & 77.92\sd{0.43} & 63.23\sd{0.42} & 63.23\sd{0.42} & 72.82\sd{0.42} & 74.42\sd{0.52} 
& 84.28\sd{0.58} & 84.15\sd{0.52} & 82.53\sd{0.65} & 81.87\sd{0.62} & 79.28\sd{0.40} & 76.40\sd{0.28} \\

IceBerg      
& \secondbest{80.99}\sd{0.83} & \secondbest{80.95}\sd{0.77} & \secondbest{63.27}\sd{1.05} & \secondbest{63.33}\sd{1.19} & 76.85\sd{1.12} & \secondbest{77.45}\sd{0.77} 
& \secondbest{85.36}\sd{0.63} & \secondbest{85.18}\sd{0.60} & \secondbest{83.64}\sd{1.40} & 81.54\sd{1.23} & 80.18\sd{0.52} & \secondbest{76.93}\sd{0.48} \\

\midrule
\textbf{GraphSB} 
& \best{81.90}\sd{0.72} & \best{81.78}\sd{0.63} & \best{68.42}\sd{0.52} & \best{67.82}\sd{0.42} & \best{80.82}\sd{1.13} & \best{80.72}\sd{0.82} 
& \best{88.63}\sd{0.45} & \best{88.55}\sd{0.46} & \best{85.87}\sd{0.36} & \best{84.87}\sd{0.34} & \best{82.46}\sd{0.37} & \best{79.40}\sd{0.33} \\
\bottomrule
\end{tabular*}
\end{table*}

We adopt the public splits from~\cite{13} for the citation networks. For Amazon-Photo, Amazon-Computers, and ogbn-arxiv, we employ a consistent partition strategy, randomly selecting 20, 25, and 55 nodes per class for training, validation, and testing, respectively. Following~\cite{zhao2021graphsmote}, we construct step-imbalanced training sets by designating the last $\lceil C/2 \rceil$ classes as minority ($\mathcal{C}_{\text{mi}}$); majority classes retain 20 nodes, while minority nodes are downsampled to $20 \times \rho$. For naturally imbalanced datasets, we follow dataset-specific protocols: Wiki-CS adopts the 1:1:2 ratio from~\cite{13}, and CoraFull follows the 10\%/40\%/50\% split strategy from~\cite{Iceberg}.

\subsubsection{Hyperparameters and Baselines}
\label{sec4.1.2}


For {\SB}, we set $p_\text{drop}=0.1$. 
For the rest, we follow the same hyperparameter settings as GraphMixup~\cite{13}.
We ran each set of experiments five times with different random seeds and reported the average results.

We adpot 11 baselines from both data-level and algorithm-level, including Vanilla~\cite{hamilton2017inductive}, Re-weight~\cite{yuan2012sampling+}, SMOTE~\cite{chawla2002smote}, Embed-SMOTE~\cite{ando2017deep}, DR-GNN~\cite{DR-GCN}, TAM~\cite{TAM}, GraphSMOTE~\cite{zhao2021graphsmote}, GraphENS~\cite{9}, GraphMixup~\cite{13}, GraphSHA~\cite{GraphSHA}, CDCGAN~\cite{liu2025cdcgan}, and IceBerg~\cite{Iceberg}.



\begin{table}[!b]
\centering
\scriptsize
\setlength{\tabcolsep}{2.5mm}
\caption{AUC under varying training imbalance ratio $\rho$.}
\vspace{-5pt}
\label{tab:imbalance}
\begin{tabular}{lcccccc}
\toprule
\multirow{2}{*}{Methods} & \multicolumn{5}{c}{Training Imbalance Ratio $\rho$} \\
\cmidrule(lr){2-6}
 & \textbf{0.1} & \textbf{0.3} & \textbf{0.5} & \textbf{0.7} & \textbf{0.9} \\
\midrule
Vanilla          & 0.843 & 0.907 & 0.912 & 0.923 & 0.925 \\
Re-weight        & 0.869 & 0.921 & 0.925 & 0.929 & 0.931 \\
SMOTE            & 0.839 & 0.917 & 0.925 & 0.930 & 0.932 \\
Embed-SMOTE      & 0.870 & 0.906 & 0.918 & 0.927 & 0.929 \\
DR-GNN           & 0.890 & 0.921 & 0.933 & 0.937 & 0.939 \\
TAM              & 0.895 & 0.925 & 0.937 & 0.941 & 0.943 \\
GraphSMOTE       & 0.887 & 0.923 & 0.931 & 0.935 & 0.937 \\
GraphENS         & 0.894 & 0.925 & 0.945 & 0.939 & 0.941 \\
GraphMixup       & 0.903 & 0.931 & 0.943 & 0.945 & 0.947 \\
GraphSHA         & 0.897 & 0.928 & 0.947 & 0.948 & 0.950 \\
CDCGAN           & 0.910 & 0.926 & 0.945 & 0.950 & 0.949 \\
IceBerg          & \secondbest{0.920} & \secondbest{0.935} & \secondbest{0.951} & \secondbest{0.955} & \secondbest{0.956} \\
\midrule
GraphSB          & \best{0.945} & \best{0.953} & \best{0.964} & \best{0.966} & \best{0.967} \\
\bottomrule
\end{tabular}
\end{table}

\subsection{Performance Comparison}
\label{sec5.2}


Table~\ref{tab:merged_results_subscript} presents comprehensive results on six benchmark datasets with Accuracy (Acc) and Macro-F1-score metrics. GraphSB consistently outperforms all baselines across all datasets. For instance, GraphSB achieves $80.82\%$ on PubMed and $88.63\%$ on Photo, surpassing the second-best method by $3.71\%$ and $3.27\%$, respectively. Moreover, the performance gain of GraphSB is also pronounced on naturally imbalanced datasets, demonstrating robustness across diverse imbalance scenarios. Additional AUC results and detailed convergence analysis are provided in Appendix~\ref{app:experiments}.

\subsection{Robustness to Training Imbalance Ratio}
\label{sec5.3}

To systematically evaluate the robustness of GraphSB under different levels of class imbalance, we conduct experiments on Cora by adjusting the training imbalance ratio $\rho \in \{0.1, 0.3, 0.5, 0.7, 0.9\}$. 
Results are presented in Table~\ref{tab:imbalance}, where GraphSB achieves consistently superior performance with AUC scores ranging from 0.945 to 0.967, outperforming IceBerg by 2.5\% under extreme imbalance ($\rho=0.1$) and maintaining the advantage across all imbalance ratios.

\subsection{Ablation study}
\label{sec5.4}
In this subsection, we evaluate the effectiveness of {\SE} (SE) and {\RD} (RD). 
Figure~\ref{Ablation study} depicts the outcomes across three datasets, where both modules contribute significantly to the performance.
It may be noticed that SE yields marginal improvements over the base model on Cora and Wiki-CS. 
This is primarily because these datasets are not severely imbalanced, where only a few additional edges are sufficient to enhance their structure. 
In contrast, SE demonstrates substantial gains on CoraFull, a network with severe imbalance.
This highlights SE's ability to adaptively compensate for the imbalanced graph structures.

\begin{figure}[!h]
\centering
\includegraphics[width=\linewidth]{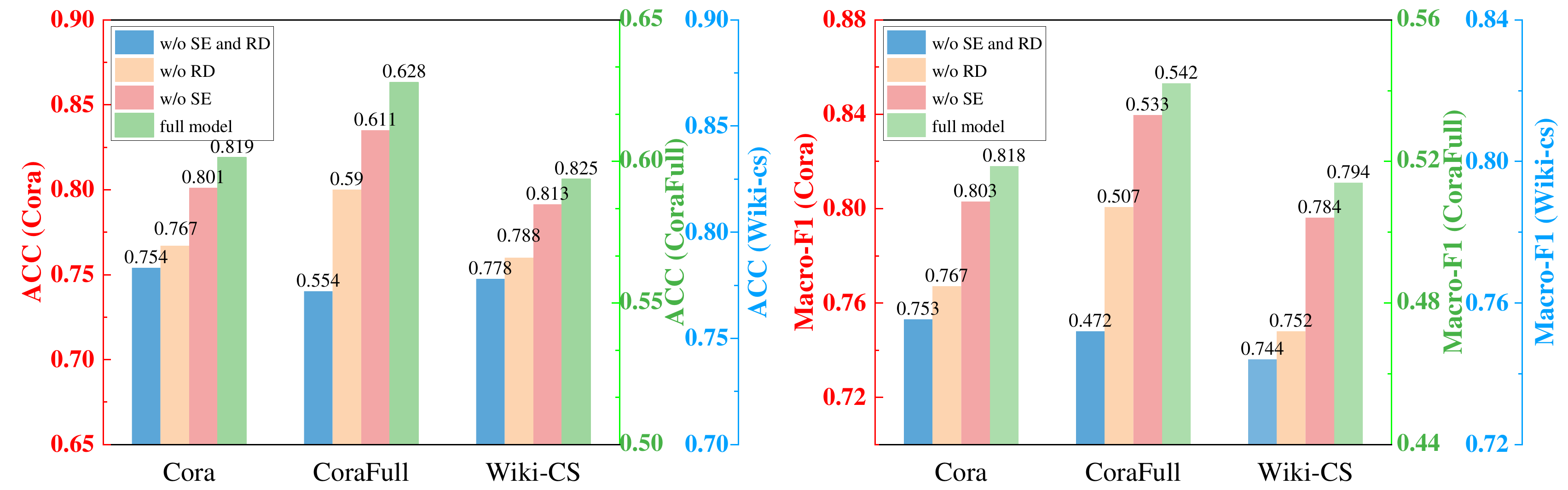}
\caption{Ablation study on SE and RD.}
\label{Ablation study}
\end{figure}




\subsection{Visualization Analysis}
\label{sec5.5}

Visualizations are conducted by projecting learned node embeddings into a 2-dimensional space via t-SNE~\cite{tsne} on PubMed. 
The perplexity is set to 30 with 5000 iterations for stability, as illustrated in Figure~\ref{fig4}.
Although Embed-SMOTE, GraphSMOTE, and GraphMixup demonstrate enhanced performance through node augmentations, the boundaries between different classes remain ambiguous. 
GraphENS enhances separability but yields dispersed minority nodes.
In contrast, GraphSB exhibits superior representation learning capability by successfully constructing three distinct and compact clusters,demonstrating its effectiveness in handling class imbalance.

\begin{figure}[!h]
\centering
\includegraphics[width=\linewidth]{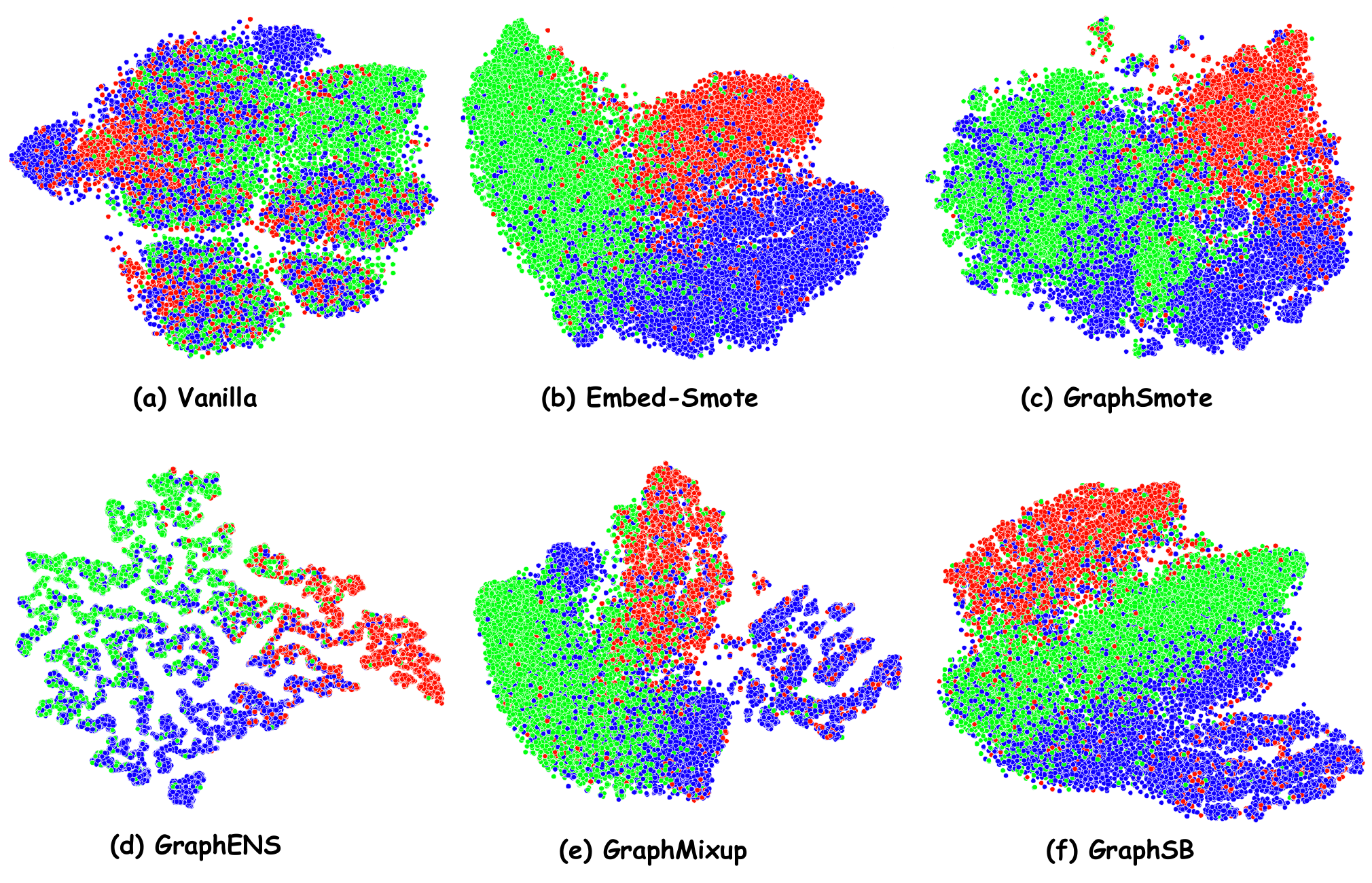}
\caption{Visualization on PubMed.}
\label{fig4}
\end{figure}


\subsection{Hyperparameter Analysis}
\label{sec5.6}
We analyze the impact of two key hyperparameters in {\RD}: the diffusion coefficient $\alpha$ and the diffusion step $K$. Figure~\ref{fig:parameter} presents the performance under different configurations.
As shown, the optimal performance is achieved at $K=10$ and $\alpha=0.15$. 
Increasing $K$ beyond 10 leads to performance degradation, indicating that excessive diffusion steps introduce noise. 
Similarly, both too small or too large $\alpha$ values compromise performance, as they either limit information propagation or cause over-smoothing.

\begin{figure}[!h]
\centering
\includegraphics[width=0.95\linewidth]{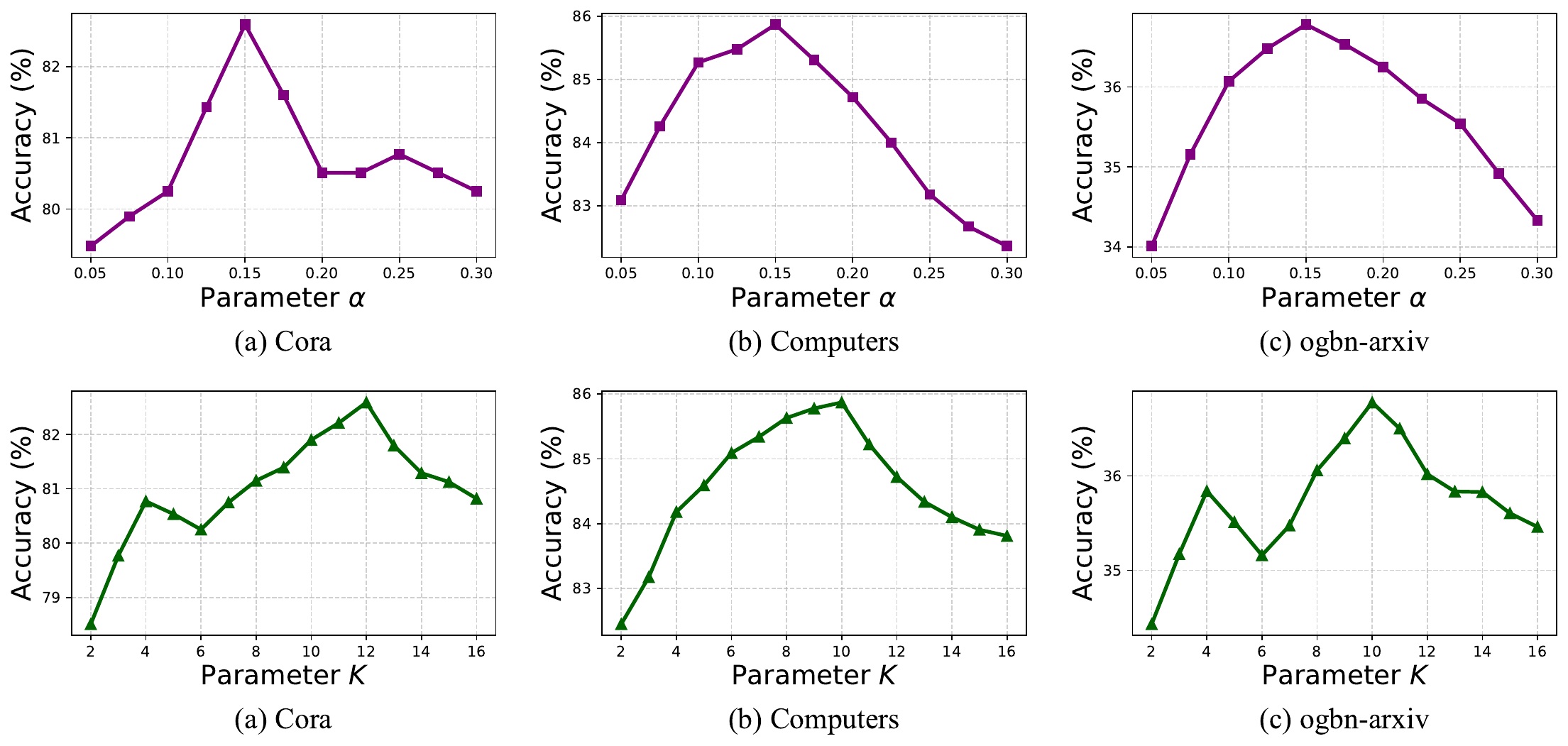}
\caption{Impact of diffusion steps $K$ and coefficient $\alpha$.}
\label{fig:parameter}
\end{figure}


\subsection{Universality of Structural Balance}
\label{sec5.7}

Table~\ref{tab:rb_results} compares performance (Acc) gains of four synthetic baselines: Embed-SMOTE (E-SM), GraphSMOTE (G-SM), GraphENS (ENS), and GraphMixup (MIX)---equipped with different plug-and-play modules (TAM, IceBerg, and SB) on three datasets.
Results show that SB consistently achieves the highest improvements, with average Acc gains of $+3.98\%$, $+6.59\%$, and $+3.15\%$ on Cora, Citeseer, and Computers respectively, substantially outperforming TAM and IceBerg.

\begin{table}[!h]
\centering
\tiny
\setlength{\tabcolsep}{0.5pt} 
\caption{Performance (Acc) gains of synthetic baselines equipped with various plug-and-play modules (TAM, IceBerg, and SB).}
\vspace{-5pt}
\label{tab:rb_results}
\begin{tabular*}{\linewidth}{@{\extracolsep{\fill}}cc|ccccc|ccccc@{}}
\toprule
\multicolumn{2}{c|}{Metric} & \multicolumn{5}{c|}{\textbf{Acc} ($\uparrow$)} & \multicolumn{5}{c}{\textbf{Macro-F1} ($\uparrow$)} \\
\midrule
\multicolumn{2}{c|}{Method} & E-SM & G-SM & ENS & MIX & Avg($\Delta$) & E-SM & G-SM & ENS & MIX & Avg($\Delta$) \\
\midrule
\multirow{4}{*}{\rotatebox{90}{Cora}} 
& BASE      & 72.32 & 74.32 & 75.81 & 77.52 & 75.00 & 72.22 & 73.92 & 74.16 & 77.44 & 74.44 \\
& +TAM      & 73.15 & 75.42 & 76.33 & 78.12 & 75.76 \tiny{(+0.76)} & 72.95 & 74.85 & 75.06 & 78.21 & 75.27 \tiny{(+0.83)} \\
& +IceBerg       & 73.52 & 78.86 & 77.92 & 80.24 & 77.64 \tiny{(+2.64)} & 73.28 & 78.15 & 76.85 & 80.12 & 77.10 \tiny{(+2.66)} \\
& +SB       & 74.34 & 81.30 & 78.33 & 81.93 & 78.98 \tiny{(+3.98)} & 74.14 & 80.27 & 77.18 & 81.78 & 78.34 \tiny{(+3.90)} \\
\midrule
\multirow{4}{*}{\rotatebox{90}{Citeseer}} 
& BASE      & 56.74 & 57.55 & 59.55 & 63.31 & 59.29 & 56.81 & 56.71 & 58.21 & 63.12 & 58.71 \\
& +TAM      & 57.63 & 58.79 & 60.18 & 64.11 & 60.18 \tiny{(+0.89)} & 57.09 & 57.51 & 59.31 & 63.92 & 59.46 \tiny{(+0.75)} \\
& +IceBerg       & 63.01 & 62.06 & 61.88 & 66.57 & 63.38 \tiny{(+4.09)} & 62.75 & 61.91 & 61.40 & 67.01 & 63.27 \tiny{(+4.56)} \\
& +SB       & 67.86 & 64.24 & 62.84 & 68.58 & 65.88 \tiny{(+6.59)} & 67.77 & 63.15 & 61.53 & 69.14 & 65.40 \tiny{(+6.69)} \\
\midrule
\multirow{4}{*}{\rotatebox{90}{Computers}} 
& BASE      & 80.06 & 82.62 & 82.45 & 82.09 & 81.81 & 80.35 & 80.15 & 81.98 & 81.59 & 81.02 \\
& +TAM      & 80.58 & 82.94 & 82.73 & 82.45 & 82.18 \tiny{(+0.37)} & 80.67 & 80.48 & 82.21 & 81.86 & 81.31 \tiny{(+0.29)} \\
& +IceBerg       & 82.42 & 83.28 & 83.15 & 83.67 & 83.13 \tiny{(+1.32)} & 81.85 & 81.42 & 82.95 & 82.78 & 82.25 \tiny{(+1.23)} \\
& +SB       & 84.26 & 85.14 & 84.58 & 85.87 & 84.96 \tiny{(+3.15)} & 83.52 & 83.89 & 84.12 & 84.87 & 84.10 \tiny{(+3.08)} \\
\bottomrule
\end{tabular*}
\end{table}

\subsection{Performance on Large-scale Graphs}
\label{sec5.8}

Table~\ref{tab:large_scale} reports the scalability evaluation on ogbn-arxiv (169K nodes, $\rho_0{=}0.001$) and CoraFull (20K nodes, $\rho_0{=}0.016$). 
Notably, SB exhibits an asymmetric enhancement pattern: classical methods (Vanilla, Re-weight (RW), and Embed-SMOTE) achieve substantial absolute gains ($+14.58\%$ on ogbn-arxiv, $+8.12\%$ on CoraFull on average), while recent sophisticated baselines (GraphENS, GraphMixup) show relatively marginal improvements. 
This disparity induces a performance convergence phenomenon---diverse baselines, once equipped with SB, attain a comparable accuracy plateau, suggesting that structural balance rectification constitutes the predominant bottleneck in extreme imbalance scenarios.

\begin{table}[!h]
\centering
\scriptsize
\setlength{\tabcolsep}{2.5mm}
\caption{Performance on large-scale graphs (Acc).}
\vspace{-5pt}
\label{tab:large_scale}
\begin{tabular}{cc|ccccc}
\toprule
\multicolumn{2}{c|}{Metric} & \multicolumn{5}{c}{\textbf{Acc} ($\uparrow$)} \\
\midrule
\multicolumn{2}{c|}{Method} & Vanilla & RW & E-SM & ENS & MIX \\
\midrule
\multirow{4}{*}{\rotatebox{90}{ogbn-arxiv}} 
& BASE      & 21.85 & 22.69 & 21.38 & OOM & OOM \\
& +TAM      & 25.28 & 27.54 & 25.12 & OOM & OOM \\
& +IceBerg      & 32.13 & 32.97 & 21.86 & OOM & OOM \\
& +SB       & 36.43 & 36.06 & 36.07 & OOM & OOM \\
\midrule
\multirow{4}{*}{\rotatebox{90}{CoraFull}} 
& BASE      & 54.89 & 55.38 & 56.02 & 55.21 & 53.99 \\
& +TAM      & 54.93 & 50.01 & 55.47 & 55.79 & 55.37 \\
& +IceBerg  & 56.49 & 59.79 & 58.96 & 59.03 & 58.83 \\
& +SB       & 63.84 & 63.76 & 63.71 & 60.83 & 62.67 \\
\bottomrule
\end{tabular}
\end{table}



\vspace{-8pt}
\section{Conclusion}
\label{sec6}

This paper proposes GraphSB, a novel framework designed to tackle imbalanced node classification. 
The core of GraphSB is {\SB} (SB), which explicitly addresses the inherent imbalanced graph structure, a root cause that disadvantages minority-class nodes as revealed by our theoretical analysis.
SB performs a two-stage structure optimization process: {\SE} builds similarity-based edges to strengthen minority-class connectivity, and {\RD} propagates minority context while capturing higher-order structural dependencies. 
Experiments show that GraphSB consistently outperforms state-of-the-art baselines. 
Furthermore, SB also offers plug-and-play compatibility that remarkably improves existing methods.

\bibliographystyle{named}
\bibliography{ijcai26}

\newpage
\clearpage

\appendix

\section{Proofs}
\label{app:proofs}
\theoremone*
\begin{proof}
The class-conditional expectations for message propagation are characterized as follows: (1) Majority-class paths ($V_\text{ma}$): $\mathbb{E}\left[\frac{\|W^{(L)}\|}{|V_\text{ma}|}\right] = \frac{\|W^{(L)}\|}{\gamma\mathbb{E}[d_\text{mi}]}$, and (2) Minority-class paths ($V_\text{mi}$): $\mathbb{E}\left[\frac{\|W^{(L)}\|}{|V_\text{mi}|}\right] = \frac{\|W^{(L)}\|}{\mathbb{E}[d_\text{mi}]}$.
The probability of randomly selecting a node from each class is:
\begin{equation}
\begin{aligned}
p_1 = P(v \in V_\text{mi}) = \frac{1}{\beta+1}, p_2= P(v \in V_\text{ma}) = \frac{\beta}{\beta+1}.
\end{aligned}
\end{equation}

For a path of length $L$, let $k_p$ represent the number of steps passing through majority-class nodes. 
By applying the binomial theorem, the path weight expectation is:
\begin{equation}
\begin{aligned}
\mathcal{W}_L &= \frac{W^{(L)}}{\mathbb{E}[d]^L} \sum_{k_p=0}^{L} \binom{L}{k_p} \left(\frac{p_2}{\gamma}\right)^{k_p} p_1^{L-k_p}\\
&= \frac{W^{(L)}}{\mathbb{E}[d_\text{mi}]^L} \sum_{k_p=0}^L \binom{L}{k_p} \left(\frac{\beta}{\gamma(\beta+1)}\right)^{k_p} \left(\frac{1}{\beta+1}\right)^{L-k_p}\\
&= \frac{W^{(L)}}{\mathbb{E}[d_\text{mi}]^L} \left(\frac{\beta + \gamma}{\gamma(\beta+1)}\right)^L.
\end{aligned}
\end{equation}
Thus concludes the proof.
\end{proof}

\theoremtwo*
\begin{proof}
The expected number of nodes at the $\ell$-th level is:
\begin{equation}\label{eq7}
\begin{aligned}
\mathbb{E}[N^{(\ell)}] &= \mathbb{E}[d_\text{mi}]^{\ell}\sum_{k_p=0}^{\ell} \binom{\ell}{k_p} \left({p_2}{\gamma}\right)^{k_p} p_1^{\ell-k_p}\\
&= \mathbb{E}[d_\text{mi}]^\ell \cdot \left(\frac{1 + \beta\gamma}{\beta+1}\right)^\ell.
\end{aligned}
\end{equation}

Class contributions can be expressed as: (1) Minority: $\mathcal{W}_L \times \frac{1}{\beta+1} \times \mathbb{E}[N^{(\ell)}]$ and (2) Majority: $\mathcal{W}_L \times \frac{\beta}{\beta+1} \times \mathbb{E}[N^{(\ell)}]$.
Combining Eq.(\ref{eq7}), we can obtain:

\begin{align}
\frac{\mathbb{E}[\|\nabla_\theta L(u)\|]}{\mathbb{E}[\|\nabla_\theta L(v)\|]} &\propto \frac{\text{Minority }}{\text{Majority }} 
= \frac{1}{\beta}.
\end{align}

Thus concludes the proof.
\end{proof}

\theoremthree*
\begin{proof}
In GNNs, a node's embedding at the $\ell$-th layer is:
\begin{equation}
h_u^{(\ell)} = W^{(\ell)} \sum_{v \in \mathcal{N}(u)} \frac{h_v^{(\ell-1)}}{\sqrt{d_u d_v}},
\end{equation}
where $\mathcal{N}(u)$ denotes the neighbors of a node $u$, $d_v$ is the degree of node $v$, and $W^{(\ell)}$ is the layer-specific weight matrix.

Then, the minority class centroids are given by
\begin{equation}
\mathbb{E}[h_u^{(\ell)}] = W^{(\ell)} \left( \frac{p}{p+q\beta} \mu_\text{mi} + \frac{q\beta}{\sqrt{(p+q\beta)(q+p\beta)}} \mu_\text{ma} \right).
\end{equation}

When overestimating the minority-class influence (via $q+p\beta>\sqrt{(p+q\beta)(q+p\beta)}$), the expected update for $u \in V_\text{mi}$ is:
\begin{equation}
\mathbb{E}[h_u^{(\ell)}] \approx W^{(\ell)} \left( \frac{p}{p + q \beta} \mu_\text{mi}^{(\ell-1)} + \frac{q \beta}{q + p \beta} \mu_\text{ma}^{(\ell-1)} \right).
\end{equation}

Similarly, for $v \in V_\text{ma}$ :
\begin{equation}
\mathbb{E}[h_v^{(\ell)}] \approx W^{(\ell)} \left( \frac{q}{p + q \beta} \mu_\text{mi}^{(\ell-1)} + \frac{p \beta}{q + p \beta} \mu_\text{ma}^{(\ell-1)} \right).
\end{equation}

Let $\mathbf{z}^{(\ell)} = [\mu_\text{mi}^{(\ell)}, \mu_\text{ma}^{(\ell)}]^\top$ denote the state vector, the GNN's message passing induces:
\begin{equation}
\mathbf{z}^{(\ell)} = M W \mathbf{z}^{(\ell-1)}, \quad 
M = \begin{bmatrix}
\frac{p}{p + q\beta} & \frac{q\beta}{q + p\beta} \\
\frac{q}{p + q\beta} & \frac{p\beta}{q + p\beta}
\end{bmatrix}.
\end{equation}

After $\ell$ layers, it becomes: 
\begin{equation}
\mathbf{z}^{(\ell)} = (WM)^\ell \mathbf{z}^{(0)}
\end{equation}

The convergence rate (the spectral radius of $WM$) is governed by $\sigma_{\max}(W)\lambda_2(M)$, where $\lambda_2(M)$ is the subdominant eigenvalue of $M$. 
Then, the centroid distance $\Delta^{(\ell)} = \mu_\text{mi}^{(\ell)} - \mu_\text{ma}^{(\ell)}$ decays as:
\begin{equation}
\|\Delta^{(\ell)}\| \leq C \cdot \underbrace{(\sigma_{\max}(W)\lambda_2(M))^\ell}_{\text{exponential decay}} \|\Delta^{(0)}\|,
\end{equation}
where the condition $\sigma_{\max}(W)\lambda_2(M)<1$ ensures progressive assimilation of node features of minority-class. 

Thus concludes the proof.
\end{proof}

\newpage
\section{Additional Experiments}
\label{app:experiments}

\begin{table*}[!t]
\centering
\footnotesize
\setlength{\tabcolsep}{3pt}
\caption{Performance comparison on six datasets (AUC).}
\label{tab:merged_auc}
\begin{tabular*}{\textwidth}{@{\extracolsep{\fill}}lcccccc@{}}
\toprule
\multirow{2}{*}{Methods} & \multicolumn{6}{c}{\textbf{AUC}} \\
\cmidrule(lr){2-7}
 & Cora & Citeseer & PubMed & Photo & Computers & Wiki-CS \\
\midrule
Vanilla      & 91.23$\pm$0.52 & 81.73$\pm$0.42 & 85.42$\pm$0.73 & 94.15$\pm$0.64 & 94.82$\pm$0.58 & 94.05$\pm$0.35 \\
Re-weight    & 92.63$\pm$0.32 & 82.84$\pm$0.32 & 87.04$\pm$0.05 & 96.18$\pm$0.50 & 96.48$\pm$0.23 & 93.98$\pm$0.49 \\
SMOTE        & 92.53$\pm$0.44 & 82.92$\pm$0.33 & 87.73$\pm$0.22 & 95.85$\pm$0.46 & 96.45$\pm$0.25 & 94.54$\pm$0.32 \\
Embed-SMOTE  & 91.82$\pm$0.54 & 83.04$\pm$0.42 & 86.32$\pm$0.24 & 93.99$\pm$0.15 & 96.16$\pm$0.14 & 94.32$\pm$0.33 \\
DR-GNN       & 93.32$\pm$0.63 & 83.23$\pm$0.72 & 86.12$\pm$0.57 & 94.52$\pm$0.47 & 95.28$\pm$0.52 & 95.03$\pm$0.19 \\
TAM          & 93.74$\pm$0.51 & 87.24$\pm$0.79 & 87.42$\pm$0.61 & 95.03$\pm$0.41 & 95.74$\pm$0.38 & 95.26$\pm$0.22 \\
GraphSMOTE   & 93.12$\pm$0.62 & 86.73$\pm$0.71 & 86.03$\pm$0.55 & 94.66$\pm$0.24 & 95.59$\pm$0.38 & 95.54$\pm$0.24 \\
GraphENS     & 94.52$\pm$0.32 & 87.53$\pm$0.42 & 87.82$\pm$0.12 & 95.72$\pm$0.35 & 96.18$\pm$0.29 & 95.38$\pm$0.37 \\
GraphMixup   & 94.32$\pm$0.69 & 88.42$\pm$0.52 & 86.62$\pm$0.59 & 95.86$\pm$0.51 & 96.81$\pm$0.08 & \secondbest{96.47$\pm$0.37} \\
GraphSHA     & 94.72$\pm$0.32 & 88.35$\pm$0.32 & 89.67$\pm$1.05 & 96.24$\pm$0.38 & 96.53$\pm$0.31 & 95.93$\pm$0.28 \\
CDCGAN       & 94.52$\pm$0.42 & 84.23$\pm$0.42 & 88.52$\pm$0.12 & 96.08$\pm$0.44 & 96.42$\pm$0.36 & 95.70$\pm$0.33 \\
IceBerg      & \secondbest{95.13$\pm$0.61} & \secondbest{85.37$\pm$0.69} & \secondbest{89.67$\pm$1.05} & \secondbest{97.50$\pm$0.23} & \secondbest{97.24$\pm$0.43} & 96.35$\pm$0.41 \\
\midrule
\textbf{GraphSB} & \best{96.43$\pm$0.42} & \best{92.82$\pm$0.42} & \best{93.72$\pm$0.52} & \best{98.18$\pm$0.21} & \best{98.05$\pm$0.08} & \best{97.74$\pm$0.08} \\
\bottomrule
\end{tabular*}
\end{table*}

\subsection{Convergence Analysis}
\label{app:convergence}
To further evaluate the training stability and efficiency of GraphSB, we record the test Accuracy and Macro-F1 scores on six benchmark datasets throughout the training process.
The convergence curves with a fixed random seed of 100 are illustrated in Figure~\ref{fig:history}.
It can be observed that GraphSB achieves rapid convergence and maintains stable performance across all datasets.

\begin{figure}[h]
\centering
\includegraphics[width=\linewidth]{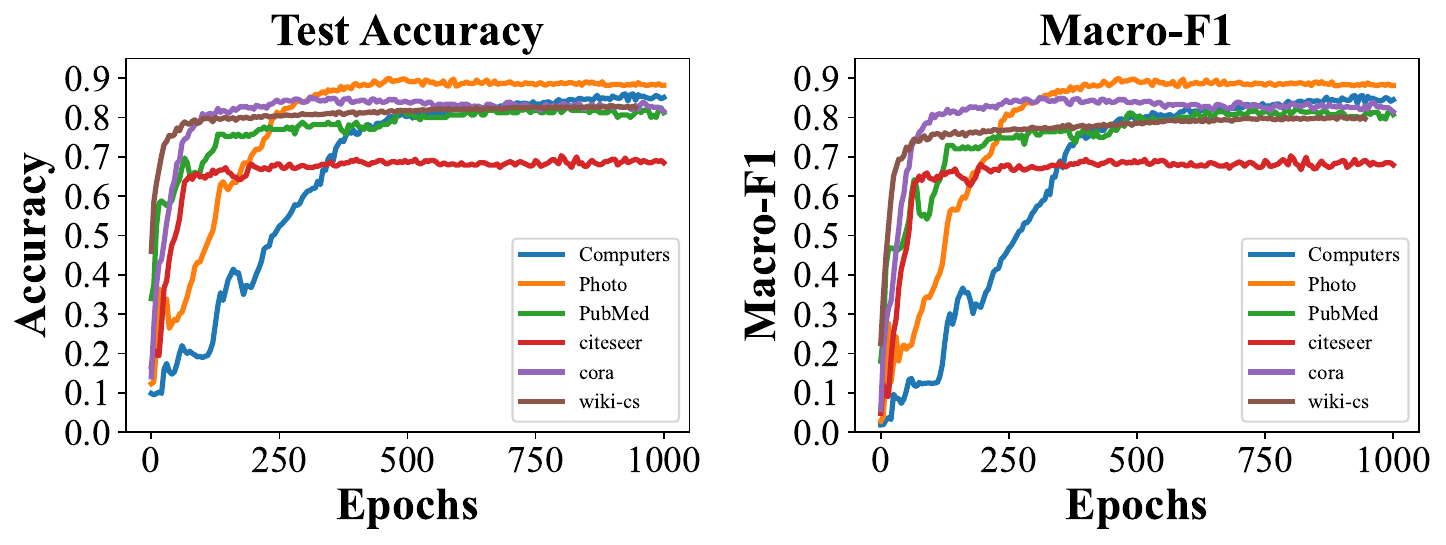}
\caption{Convergence curves of Acc and Macro-F1 on six datasets (Seed=100).}
\label{fig:history}
\end{figure}

\subsection{Universality of Structural Balance}

\subsubsection{Visualization}

Table~\ref{tab:rb_results}  presents the performance of various methods with and without the proposed {\SB} (SB) module.
To further illustrate the effectiveness of integrating SB into existing state-of-the-art approaches, we also provide a corresponding visualization on the Cora dataset.
As shown in Figure~\ref{fig5}, incorporating the SB module leads to clearer class boundaries, indicating a more discriminative and well-structured embedding space compared to the original methods.


\begin{figure}[!h]
\centering
\includegraphics[width=\linewidth]{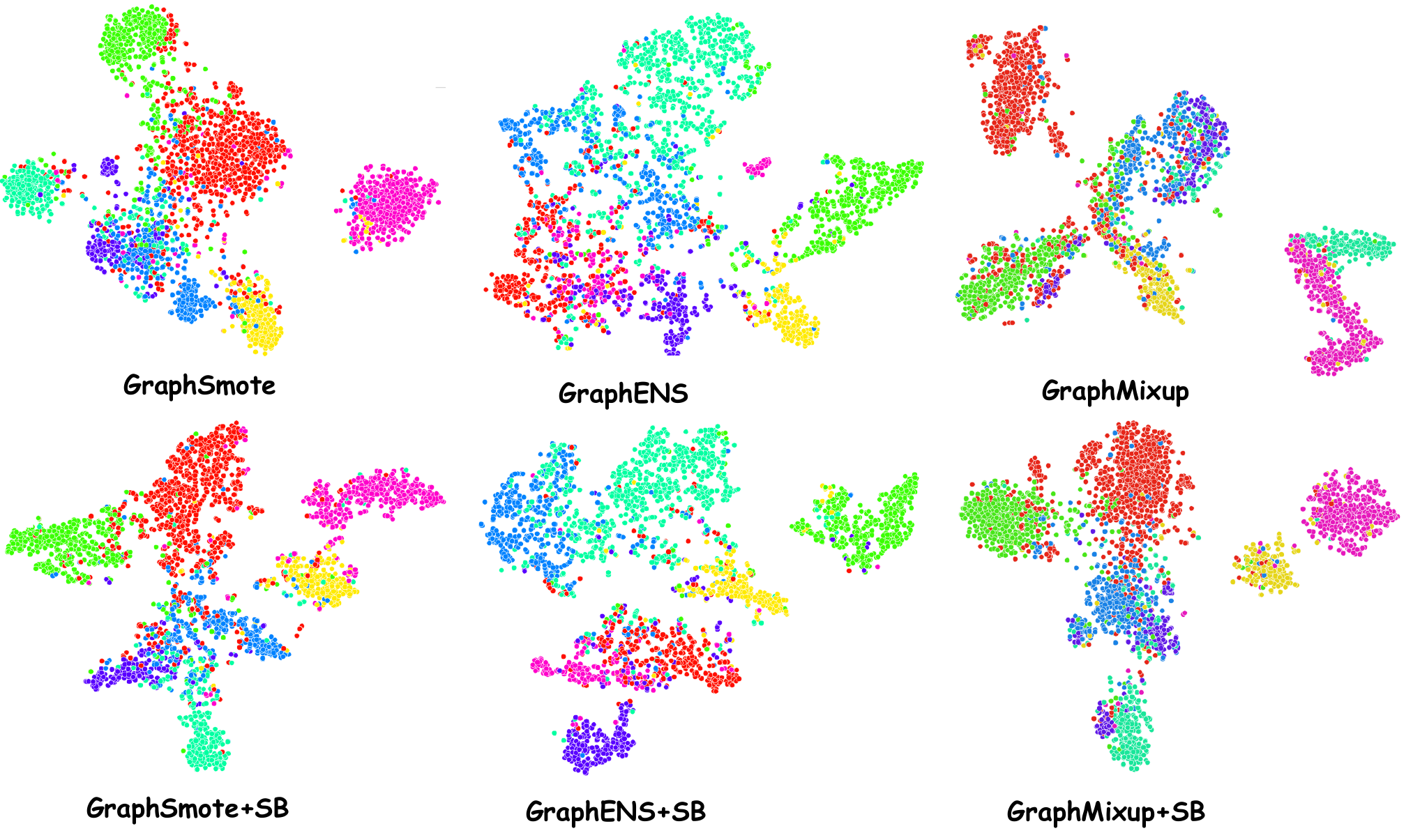}
\caption{Visualization of GraphSmote, GraphENS, GraphMixup with and without SB on Cora.}
\label{fig5}
\end{figure}

\subsubsection{Intra-to-inter-class Ratio}
The core mechanism of {\SB} (SB) involves initially concentrating attention on the target node, which subsequently diffuses to its neighboring nodes over multiple steps. 
This iterative diffusion process yields a continuous weighted diffusion matrix that defines the edge weights between the target node and other nodes. 
By prioritizing relational balance over numerical balance, the SB module enhances intra-class connectivity for minority-class nodes while mitigating cross-class interference from majority-class nodes. 

To evaluate the quality of the learned graph embeddings, we introduce the intra-to-inter-class distance ratio as a novel metric to quantify both intra-class consistency and inter-class separability. 
For a category \( C_k \) containing \( N_k \) samples, the intra-class distance is defined as:  
\[
D_{\text{intra}}(C_k) = \frac{1}{N_k (N_k - 1)} \sum_{i \neq j} d(x_i, x_j),
\]  
where \( d(x_i, x_j) \) represents the Euclidean distance between the embeddings of nodes \( x_i \) and \( x_j \) within the same category.

In contrast, the inter-class distance between two distinct categories \( C_i \) and \( C_j \) is defined as:  
\[
D_{\text{inter}}(C_i, C_j) = \frac{1}{N_i N_j} \sum_{x_i \in C_i} \sum_{x_j \in C_j} d(x_i, x_j).
\]  

Finaly, we define the intra-to-inter-class distance ratio as:
\begin{equation}
R = \frac{\frac{1}{|C|} \sum_{k} D_{\text{inter}}(C_k, C_{\bar{k}})}{\frac{1}{|C|} \sum_{k} D_{\text{intra}}(C_k)},
\label{eq:ratio}
\end{equation}
where \( C_{\bar{k}} \) represents the set of all categories except \( C_k \). A higher value of \( R \) indicates that the embeddings exhibit greater inter-class separation relative to intra-class compactness, which is favorable for improving downstream classification performance.  

Figure~\ref{fig2} exhibits the intra-to-inter-class distance ratio for various methods, both with and without the proposed SB module.
Evidently, the SB module consistently enhances this ratio across all methods, with the most significant improvement observed on GraphSMOTE (+0.42).
This gain can be attributed to the SB module's ability to mitigate the impact of inherent graph imbalance, thereby promoting better inter-class separation while preserving intra-class compactness.

\begin{figure}[!h]
\centering
\includegraphics[width=0.85\linewidth]{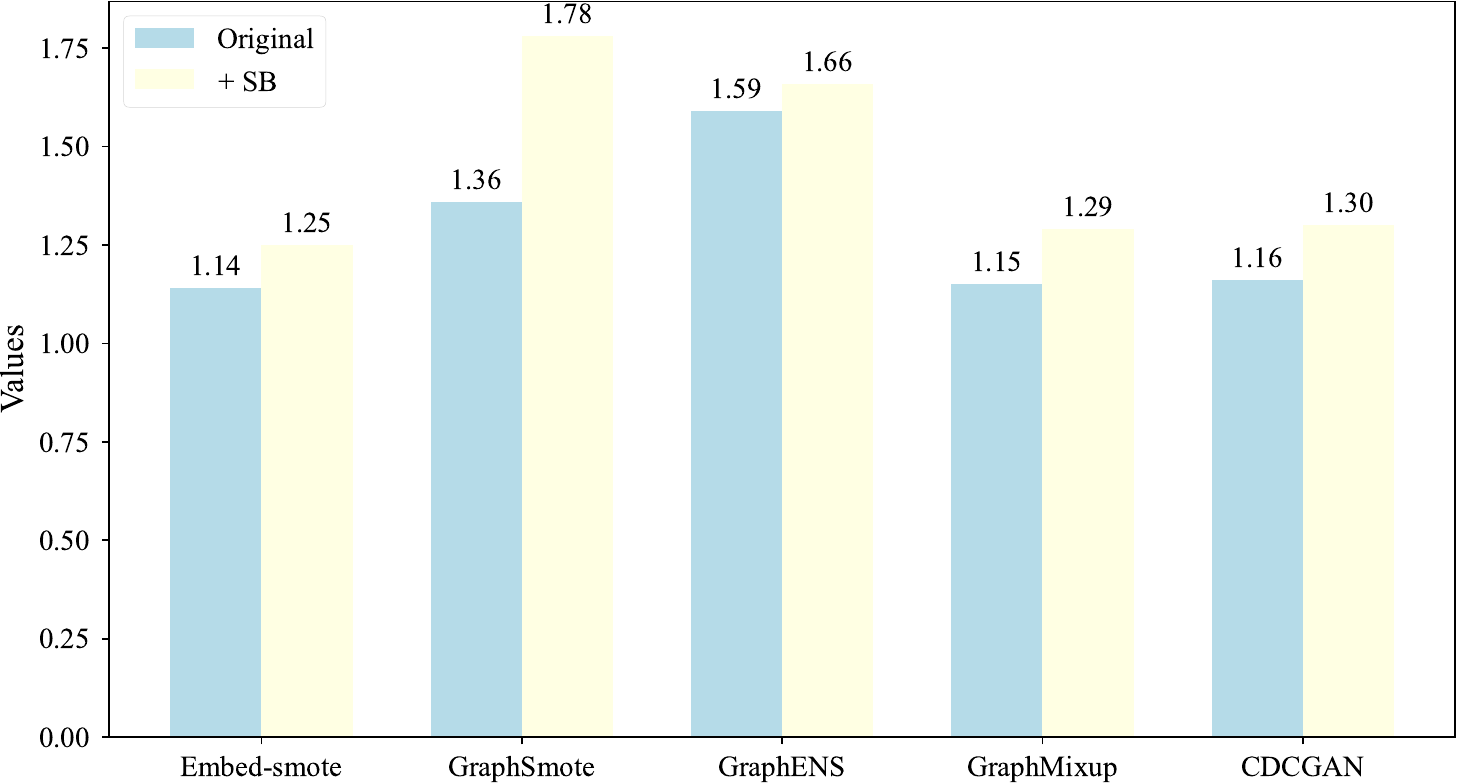}
\caption{Impact of SB on the intra-to-inter-class distance ratio.}
\label{fig2}
\end{figure}

\subsection{Over-Sampling Scale}

In the Quantity Balance step, as described in Figure~\ref{framework}, we employ a reinforcement learning module that adaptively addresses the issue of over-sampling scale. 
That is to say, for different datasets, GraphSB automatically determines appropriate over-sampling scales for each minority class.
Starting from an initial ratio $\pi_{\text{init}}^i$ for each dataset $i$, the final over-sampling scale is learned through $\Delta\pi = \pi^i-\pi_{\text{init}}^i$. 
Moreover, even on the same dataset, GraphSB can generate different numbers of supplementary nodes under different base architectures, as the over-sampling scale is learned adaptively. Figure~\ref{fig:rl_adaptive} demonstrates the over-sampling ratio adjustment process across four datasets.

\begin{figure}[!h]
\centering
\includegraphics[width=0.85\linewidth]{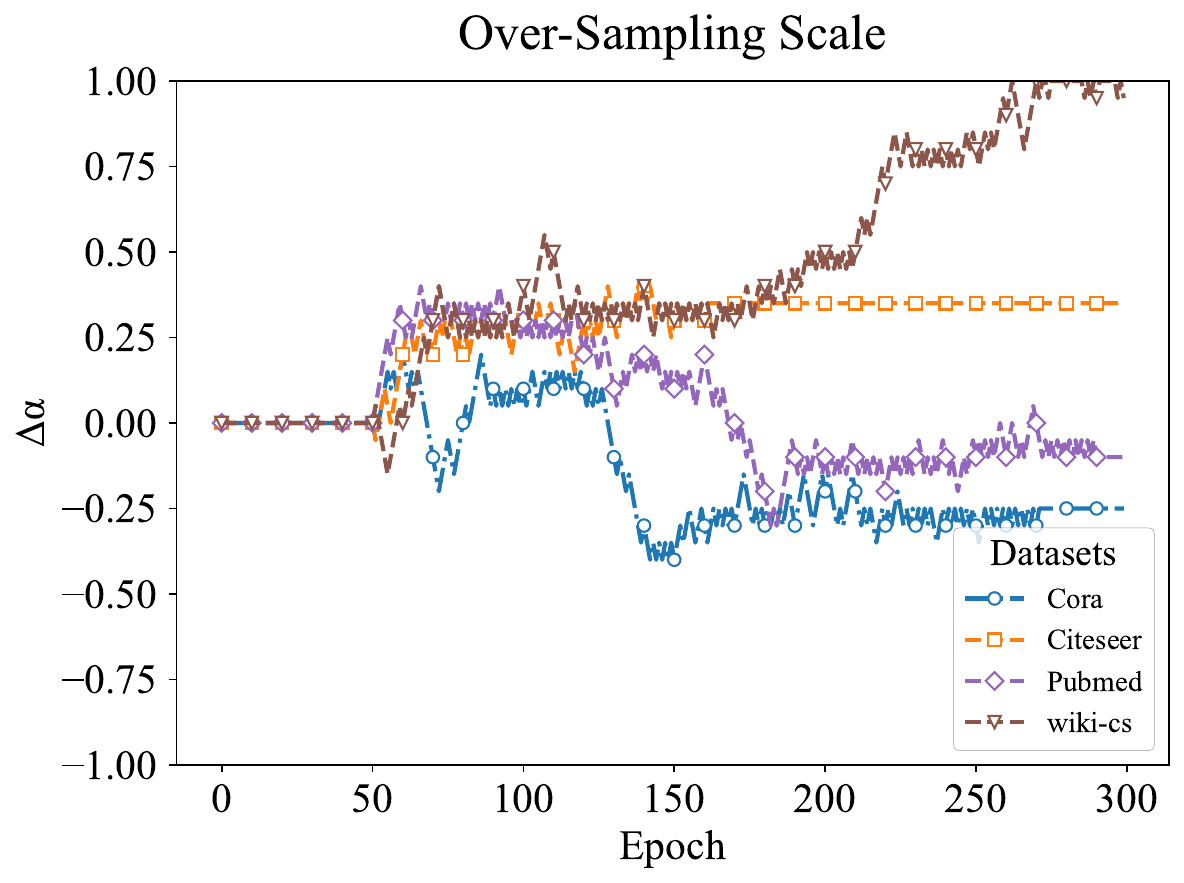}
\caption{Changes in $\Delta\pi$ across four datasets.}
\label{fig:rl_adaptive}
\end{figure}

Clearly, each dataset benefits from a different over-sampling scale, which supports our claim that the reinforcement learning module can adaptively determine the appropriate scale based on datasets.

\subsection{Robustness to Training Imbalance Ratios}

\begin{figure}[!h]
\centering
\includegraphics[width=\linewidth]{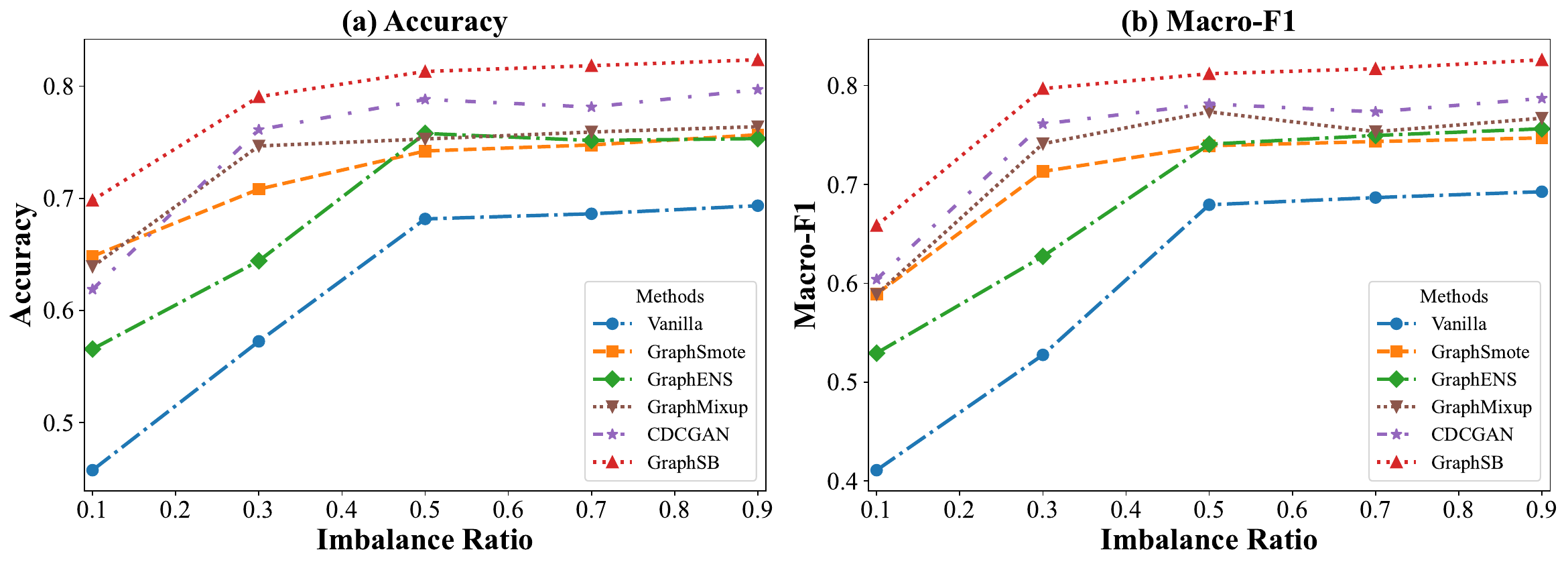}
\caption{Accuracy and Macro-F1 results under varying imbalance ratios.}
\label{fig:imbalance ratios}
\end{figure}

Table~\ref{tab:imbalance} exhibits the results of AUC under varying imbalance ratios on Cora. 
To systematically evaluate the robustness of GraphSB under different levels of class imbalance, we also report Accuracy and Macro-F1-score results with different training imbalance ratios $\rho \in \{0.1, 0.3, 0.5, 0.7, 0.9\}$. 
As shown in Figure~\ref{fig:imbalance ratios}, GraphSB demonstrates strong performance in terms of both Accuracy and Macro-F1-score across different imbalance ratios
The results further validate our model's effectiveness and robustness, particularly in handling minority classes. 
More impressively, under extreme imbalance conditions ($\rho=0.1$), GraphSB achieves significant improvements in both metrics, outperforming the second-best method by 8\% in Accuracy and 5.4\% in Macro-F1-score.

\section{Complexity Analysis}
\label{app:complexity}
The proposed GraphSB maintains a linear computational complexity comparable to standard GCNs. 
Let $|\mathcal{V}|$, $|E|$, and $d$ denote the number of nodes, edges, and feature dimensions, respectively. 
For the {\SE} (SE) module, the neighbor consensus aggregation dominates with $\mathcal{O}(|E|d)$, while anchor alignment adds negligible overhead by targeting only a subset of borderline nodes. 
Furthermore, the {\RD} (RD) module replaces computationally expensive matrix inversion ($\mathcal{O}(|\mathcal{V}|^3)$) with \textbf{sparse iterative propagation}, which scales linearly as $\mathcal{O}(K|E_{aug}|d)$, where $K$ is the diffusion step. 
Consequently, the overall complexity is $\mathcal{O}(|E|d)$, ensuring scalability to large-scale graphs.

\end{document}